\documentclass{article} 
\usepackage{nips14submit_e,times}
\nipsfinalcopy

\usepackage{graphicx}
\usepackage{subfigure}
\usepackage{amsmath}
\usepackage{amsthm}
\usepackage{amsfonts}
\usepackage{color}
\usepackage[square,sort,comma,numbers]{natbib}
\usepackage{url}
\usepackage{enumerate}

\newcommand{\R}{\mathbb{R}}
\newcommand{\tr}[1]{\mathrm{Tr}\left\{ #1\right\}}
\newcommand{\step}{\alpha}
\newcommand{\stept}{\alpha_t}
\newcommand{\T}{^{\top}}
\newcommand{\rt}{r_t}
\newcommand{\ph}{\phi}
\newcommand{\pht}{\phi_t}
\newcommand{\phtnext}{\phi_{t+1}}
\newcommand{\wfp}{w^*}

\newcommand{\Exp}[1]{\mathbb{E}\left[ #1 \right]}
\newcommand{\CExp}[2]{\mathbb{E}\left[ \left. {#1} \right| {#2} \right]}

\newtheorem{definition}{Definition}
\newtheorem{theorem}{Theorem}

\newtheorem{lemma}[theorem]{Lemma}

\newcommand{\TD}{TD($\lambda$)}

\newcommand{\en}{\| e_t \|}
\newcommand{\dn}{\| d_t \|}
\newcommand{\ed}{e_t \T d_t}
\newcommand{\Qen}{\| Q_t e_t \|}
\newcommand{\Qed}{ e_t\T Q_t d_t}

\title{Implicit Temporal Differences}

\author{ Anonymous Author 1 \And Anonymous Author 2 \And Anonymous Author 3 }
\author{
Aviv Tamar\\
The Technion - Israel Institute of Technology\\
Haifa 32000, Israel \\
\texttt{avivt@tx.technion.ac.il}\\
\And
Panos Toulis\\
Department of Statistics, Harvard University\\
Cambridge, MA 02138, USA\\
\texttt{ptoulis@fas.harvard.edu}\\
\And
Shie Mannor\\
The Technion - Israel Institute of Technology\\
Haifa 32000, Israel \\
\texttt{shie@ee.technion.ac.il} \\
\And
Edoardo M. Airoldi\\
Department of Statistics, Harvard University\\
Cambridge, MA 02138, USA\\
\texttt{airoldi@fas.harvard.edu}\\
}
\begin{document}

\maketitle

\begin{abstract}
In reinforcement learning, the \TD\ algorithm is a fundamental policy evaluation method with an efficient online implementation that is suitable for large-scale problems.
One practical drawback of \TD\ is its sensitivity to the choice of the step-size. It is an empirically well-known fact that a large step-size leads to fast convergence, at the cost of higher variance and risk of instability.
In this work, we introduce the \emph{implicit \TD} algorithm which has the same function and computational cost as \TD, but is significantly more stable. We provide a theoretical explanation of this stability and an empirical evaluation of implicit \TD\ on
typical benchmark tasks.
Our results show that implicit \TD\ outperforms standard \TD\ and a state-of-the-art method that automatically tunes the step-size,
and thus shows promise for wide applicability.
\end{abstract}

\section{Introduction}

In reinforcement learning (RL), the TD($\lambda$) algorithm \cite{sutton_reinforcement_1998} is a fundamental method for value function estimation. The efficient online implementation of \TD, which has a complexity linear in the number of features, makes it particularly effective for problems with a large number of features when paired with linear function approximation. Indeed, \TD\ has been successfully applied to large-scale domains \cite{powell2007approximate}.

One practical drawback of \TD\ is its sensitivity to the choice of the step-size. It is an empirically well-known fact that a large step-size leads to faster convergence, at the cost of a higher variance and a higher risk of instability.
There have been several studies that tried to solve the step-size problem \cite{sutton1992adapting,george2006adaptive,hutter2007temporal,mahmood2012tuning} through
adaptive step-sizes or variance reduction techniques.
However, all of these methods do not deal with the stability issue
and are prone to divergence when the initial step-size is misspecified \cite{dabney2012adaptive}.
Naturally, the step-size problem is also a well-known difficulty in stochastic gradient descent (SGD) algorithms.
Recently, \emph{implicit} SGD was introduced and it was shown to be asymptotically identical to standard SGD as a statistical
estimation method, but significantly more stable in small-to-moderate samples \cite{Toulis2014sgd}.
The stability properties of implicit SGD can be motivated either through its connection to proximal methods in optimization \cite{beck2009fast, wang2013stabilization},
or through its interpretation as a shrinkage method in statistics \cite{toulis2014stochastic}.

Inspired by \cite{Toulis2014sgd}, in this work we introduce the \emph{implicit \TD} algorithm,
and show that it also enjoys better stability properties than standard \TD.
We provide a theoretical explanation of this stability by introducing a novel bound on the maximal weight change in each \TD\ iteration, and showing that for implicit \TD\ this change is suitably contained. In addition, we present an empirical evaluation of implicit \TD\ when applied within a SARSA \cite{sutton_reinforcement_1998} policy improvement algorithm, and show that it outperforms standard \TD\, and the Alpha-Bound method \cite{dabney2012adaptive} on several benchmark tasks.

\paragraph{Related Work.}
Proximal methods have been used for stabilization of stochastic iterative
solutions of a linear equation $A x = b$ where $A$ is singular or nearly-singular \cite{wang2013stabilization}. Regularization methods have also been applied for more stable, off-policy learning
\cite{meyer2012l1, liu2012regularized}. Another notable algorithm that
was introduced recently is
the gradient temporal-difference algorithm (GTD) \cite{sutton2009convergent} which is more stable
than \TD\ by keeping less-varied estimates of the TD errors through averaging.
In regard to adaptive step-sizes, the state-of-the-art Alpha-Bound method \cite{dabney2012adaptive} is
based on the heuristic that the TD error should not change sign between subsequent updates.
Our work is distinct because it combines three important properties of a learning algorithm.
First, it is simple since it is based on a simple variation of standard \TD,
and thus inherits its known convergence properties. Second, it is trivial to implement  -- see Equation \eqref{eq:imp_td_lambda_efficient} -- and has complexity that is linear in the number of features.
Third, it is stable because, as an implicit method,
it incorporates second-order information by its definition \cite{toulis2014stochastic}.

\section{Explicit and Implicit TD($\lambda$)}

In this section we review the standard \TD\ algorithm, and present the implicit \TD\ method.

\subsection{Background}
Consider a Markov reward process (MRP) in discrete time with a finite state space $X$, an initial state distribution $\xi_0$, transition probabilities $P(x'|x)$, where $x,x'\in X$, and a deterministic and bounded reward function $r:X \to \R$. We assume that the Markov chain underlying the state transitions is ergodic and uni-chain, so that it admits a stationary distribution $\xi_\infty$. We denote by $x_t$ and $r_t$ the state and reward, respectively at time $t$, where $t = 0,1,2,\cdots$.

Our goal is to learn the weights $w\in \R^k$ of an approximate value function $V(x;w)=w\T \ph(x) \approx V(x) \doteq \CExp{\sum_{t=0}^{\infty}\gamma^t r(x_t)}{x_0=x},$ where $\ph(x)\in\R^k$ is a state-dependent feature vector, and $\gamma\in (0,1)$ is a discount factor. For brevity, in the sequel we denote $\pht = \phi(x_t)$ and $\rt = r(x_t)$.

\newcommand{\wim}[1]{w_{#1}^{\mathrm{im}}}
\newcommand{\eigenim}[1]{\lambda^{\mathrm{im}, #1}}

\subsection{TD Algorithms}
\TD\ algorithms calculate the approximation weights $w$ iteratively, using sampled state transitions and rewards from the MRP.
The standard TD($\lambda$) algorithm \cite{sutton_reinforcement_1998} updates $w$ according to
\begin{equation*}\label{eq:td_lambda}
    w_{t+1} = w_t + \stept \left [ \rt + \gamma \phtnext\T w_t - \pht\T w_t \right ] e_t,
\end{equation*}
where the eligibility trace $e_t\in \R^k$ is updated according to $e_t = \gamma \lambda e_{t-1} + \pht$.
Note that since $\pht = e_t - \gamma \lambda e_{t-1},$ standard TD($\lambda$) may also be written as
\begin{equation*}\label{eq:td_lambda_ver2}
    w_{t+1} = w_t + \stept \left [ \rt + \gamma \phtnext\T w_t + \gamma \lambda e_{t-1}\T w_t - e_t\T w_t \right ] e_t.
\end{equation*}
We now introduce the implicit TD($\lambda$) algorithm. To discriminate from standard \TD, we denote the weights for implicit \TD\ by $\wim{t}$. The implicit \TD\ algorithm updates the weights as follows:
\begin{equation}\label{eq:imp_td_lambda}
  \wim{t+1} = \wim{t} + \stept \left [ \rt + \gamma \phtnext\T \wim{t} + \gamma \lambda e_{t-1}\T \wim{t} - e_t\T
\textcolor{blue}{\wim{t+1}} \right ] e_t.
\end{equation}
Note that Eq. \eqref{eq:imp_td_lambda} is implicit because $\wim{t+1}$ appears in both sides of the equation.
Using the Sherman-Morrison formula, Eq. \eqref{eq:imp_td_lambda} may be solved for $\wim{t+1}$ as follows:
\begin{equation}
\label{eq:imp_td_lambda_efficient}
   \wim{t+1} = \left(I - \frac{\stept}{1+\stept ||e_t||^2} e_t e_t\T\right) \left(\wim{t} + \stept \left [ \rt + \gamma \phtnext\T \wim{t}+ \gamma \lambda e_{t-1}\T \wim{t} \right ] e_t\right).
\end{equation}
Note that since Eq. \eqref{eq:imp_td_lambda_efficient} can be solved by only computing inner products between vectors.
Thus, the complexity for solving Eq. \eqref{eq:imp_td_lambda_efficient} is $\mathcal{O}(k)$ i.e., it has the same complexity as standard TD($\lambda$).

The \emph{fixed point} $\wfp$ of standard \TD\ satisfies
\begin{equation*}\label{eq:fixed_point}
    \wfp = \wfp + \stept \Exp{ \left ( \rt + \gamma \phtnext\T\wfp  - \pht\T\wfp  \right ) e_t },
\end{equation*}
where the expectation is taken over the stationary distribution of the states $\xi_\infty$. It is well known \cite{BT96} that for a suitably decreasing step size, standard TD($\lambda$) converges to $\wfp$. From Eq. \eqref{eq:imp_td_lambda}, it is clear that implicit TD($\lambda$) has the same fixed point, thus, when it converges, it converges to $\wfp$ as well. This shows that implicit TD($\lambda$) has the same \emph{function} as standard TD($\lambda$). However, as we shall now show, implicit TD($\lambda$) is more stable than standard TD($\lambda$).

\section{Stability Analysis}

In this section we analyze the stability of standard and implicit TD($\lambda$), and show that implicit TD($\lambda$) is more stable than standard TD($\lambda$). We consider a \emph{fixed} step size $\alpha_t = \alpha$.
For simplicity, in this section we further assume that the reward $r(x)$ is zero for all states $x\in X$. This allows us to focus on the critical cause of instability with less notational clutter. It is straightforward, however, to extend our results to the $r \neq 0$ case.
We start by introducing the following definitions.
\begin{definition}
\label{def:main}
Let
\begin{equation*}\label{eq:C_d_def}
    \begin{split}
      d_t &\doteq \pht - \gamma \phtnext, \quad X_t \doteq e_t d_t\T , \\
	  Q_t &\doteq \left(I + \alpha e_t e_t\T\right)^{-1}, \quad \beta_t \doteq 1 - \frac{\step \en^2}{1+\step \en^2}.
    \end{split}
\end{equation*}
\end{definition}
Under our assumption that $r=0$, the standard and implicit TD($\lambda$) iterations may be written as
\begin{equation}\label{eq:td_lambdas_no_reward}
\begin{split}
    w_{t+1} &= (I - \step X_t) w_t,\\
    \wim{t+1} &= (I - \step Q_t X_t) \wim{t}.
\end{split}
\end{equation}
Thus, at time $T$, we have $w_T = \prod_{t=0}^{T-1} (I - \step X_t) w_0$, and similarly, assuming $\wim{0} = w_0$, we have $\wim{T} = \prod_{t=0}^{T-1} (I - \step Q_t X_t) w_0$. Letting $\|\cdot\|_2$ denote the matrix spectral norm, it is clear that if $\|I - \step X_t\|_2 \leq 1$ for all $t$, then $\|w_{T}\|_2 \leq \prod_{t=0}^{T-1} \|I - \step X_t\|_2 \|w_0\|_2$, and the standard TD($\lambda$) iterates stay bounded. Unfortunately, there is no guarantee that such a result would occur in practice, and indeed, as we show in the experiments, the TD($\lambda$) iterates frequently diverge unless $\step$ is very small. For implicit TD($\lambda$), however, a much more stable performance was observed. We now provide a theoretical explanation for this observation.

We start with an informal argument. For this argument, assume that $X_t$ is symmetric. Then, $\|I - \step X_t\|_2 > 1$ if either $X_t$ has a negative eigenvalue, or an eigenvalue larger than $2/\alpha$. Since $Q_t$ is positive definite with eigenvalues smaller than 1, it is a contraction, and so $\|I - \step Q_t X_t\|_2 <\|I - \step X_t\|_2$, possibly preventing the divergence. In the following, we make this argument formal, without the unjustified assumption that $X_t$ is symmetric.

Our main result is the following Lemma, where we calculate $\|I - \step X_t\|_2$ and $\|I - \step Q_t X_t\|_2$.
\begin{lemma}\label{lemma:td_eigen_vals}
The matrix $(I - \stept X_t)(I - \stept X_t)\T$ has $k-2$ eigenvalues equal to $1$, and 2 eigenvalues that are given by
\begin{equation*}
    \lambda^+,\lambda^- = 1 + \frac{\step^2 \en^2 \dn^2 -2 \step \ed \pm \step \en \dn \sqrt{\step^2 \en^2 \dn^2 +4 - 4 \step \ed}}{2}.
\end{equation*}
The matrix $(I - \step Q_t X_t)(I - \step Q_t X_t)\T$ has $k-2$ eigenvalues equal to $1$, and 2 eigenvalues that are given by
\begin{equation*}
    \eigenim{+}, \eigenim{-} =1 + \frac{\step^2 \beta_t^2\en^2 \dn^2 -2 \step \beta_t\ed \pm \step \beta_t\en \dn \sqrt{\step^2 \beta_t^2\en^2 \dn^2 +4 - 4 \step \beta_t\ed}}{2}.
\end{equation*}
Furthermore, we have $\|I - \step X_t\|_2 = \max \{ \lambda^+,1 \}$, and $\|I - \step Q_t X_t\|_2 = \max \{\eigenim{+}, 1\}$.
\end{lemma}

From Lemma \ref{lemma:td_eigen_vals}, it is clear that the difference between $\|I - \step X_t\|_2$ and $\|I - \step Q_t X_t\|_2$ lies in the difference between $\en$ and $\beta_t \en$, and the effect of this difference on $\lambda^+$ and $\eigenim{+}$. But $ \beta_t \en \leq \frac{1}{\step}$, by definition, and so implicit \TD\ is a more stable procedure. In general, when $||e_t||$ is large,
or equivalently when some $||\phi_i||$, $i<t$ is large, then
the eigenvalue $\lambda^+$ for \TD\ is directly affected.
In contrast, the factor $\beta_t$ in implicit \TD\ shrinks to zero,
thus stabilizing the iteration.

\section{Experiments}

We now evaluate our implicit TD($\lambda$) method on several standard benchmarks. We use the SARSA($\lambda$) TD learning algorithm, with fixed step-size implicit TD($\lambda$) as the policy evaluation step; we term this \emph{implicit SARSA($\lambda$)}. We compare to standard SARSA($\lambda$) with a constant step-size, and to SARSA($\lambda$) with the Alpha Bound adaptive step-size of \cite{dabney2012adaptive}. In addition, we note that our implicit SARSA($\lambda$) method may be combined with any other step-size adaptation method. Here, we chose the heuristic Alpha Bound method of \cite{dabney2012adaptive}, resulting in the implicit SARSA($\lambda$) with Alpha Bounds adaptive step-size. All experiments were performed using the RLPy library \cite{RLPy}.

In Figure \ref{fig:results} (left) we show results for the puddle world domain with 3rd order Fourier features and $\lambda=0.5$. We plot the final average reward after 40,000 training steps for different initial step-sizes. It may be seen that both standard SARSA($\lambda$) and SARSA($\lambda$) with the Alpha Bound are not stable for a step size larger than 0.1. Implicit TD($\lambda$), on the other hand, is stable for all step-sizes.
In Figure \ref{fig:results}(right) we show similar results for the cart-pole domain, also with 3rd order Fourier features and $\lambda=0.5$. Similar results were obtained for other standard benchmark domains such as acrobot.

\begin{figure}[th!]
\centering
\includegraphics[scale=0.45]{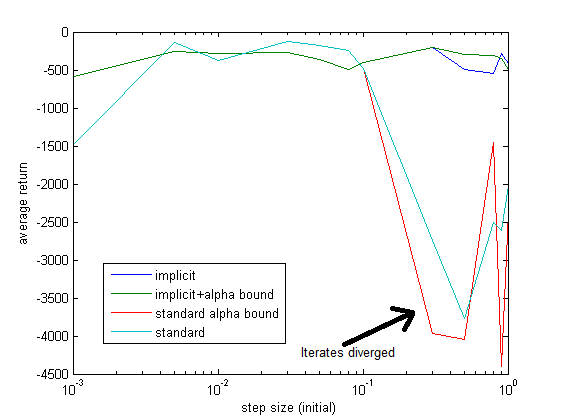}
\includegraphics[scale=0.45]{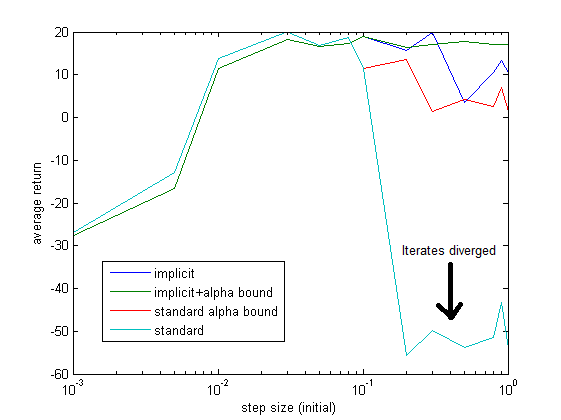}
\caption{\label{fig:results} Left: Puddle world domain. Right: Cart-pole domain.}
\end{figure}

\section{Conclusion and Outlook}
In this work we presented the implicit \TD\ algorithm, a novel \TD\ variant that is more stable, and it is justified both by statistical theory and empirical evidence. Since implicit \TD\ has the same function and the same complexity as standard \TD, it shows great promise to become the de-facto
method for \TD\ learning.

This work is still in a preliminary stage, and there are many interesting questions ahead.
First of all, we would like to have a full statistical analysis of the convergence and errors of implicit \TD, similarly to the existing results for standard \TD\ \cite{BT96}. For a suitably decreasing step-size, it can be shown that implicit \TD\ converges almost surely to the standard \TD\ fixed point.
We also have preliminary results for the bias and variance of implicit \TD, and in the future we intend to use them to derive optimal step-sizes. We also intend to conduct more extensive empirical evaluations with several adaptive step-size rules, and compare with other algorithms
based on proximal methods.

Finally, the implicit update idea has been successful in providing stability for both SGD and \TD\ algorithms. It is interesting whether this idea can be extended to other online RL methods, such as actor-critic \cite{bhatnagar_natural_2009} and policy gradients \cite{sutton_policy_2000}, which are also notoriously prone to instability when a large step-size is used.

\small
\bibliographystyle{abbrv}
\bibliography{my_library}

\newpage
\appendix
\section{Proofs}

\paragraph{Proof of Lemma \ref{lemma:td_eigen_vals}}
We first prove an intermediate result.
\begin{lemma}\label{lemma:rank_2_eigen_vals}
Let $A = a b\T + c d\T \in \R^{k \times k}$. Then A has $k-2$ eigenvalues equal to zero, and 2 eigenvalues:
\begin{equation*}
    \frac{a\T b + c\T d \pm \sqrt{(a\T b -c\T d)^2 + 4(a \T d)(b\T c)}}{2}.
\end{equation*}
\end{lemma}
\begin{proof}
Since $A$ is of rank 2, it has $k-2$ zero eigenvalues. We now calculate the two remaining eigenvalues $\lambda_1,\lambda_2$.

We have $\tr{A} = \tr{a \T b}+\tr{c \T d}=a\T b + c\T d$. Similarly, we have $\tr{A^2} = (a\T b)^2 + (c\T d)^2 + 2 (a\T d)(b\T c)$.

The only two non-zero eigenvalues of $A^2$ are $\lambda_1^2$ and $\lambda_2^2$, therefore we obtain the following equations
\begin{equation}\label{eq:eig_proof_1}
    \begin{split}
      \lambda_1+\lambda_2 &= a\T b + c\T d, \\
      \lambda_1^2+\lambda_2^2  &= (a\T b)^2 + (c\T d)^2 + 2 (a\T d)(b\T c).
    \end{split}
\end{equation}
With a little algebra, it can be shown from \eqref{eq:eig_proof_1} that
\begin{equation}\label{eq:eig_proof_2}
    \lambda_1 \lambda_2 = (a\T b)(c\T d) - (a\T d)(b\T c).
\end{equation}
From Eq. \eqref{eq:eig_proof_1} and Eq. \eqref{eq:eig_proof_2}, we obtain a quadratic equation for $\lambda_1 \lambda_2$, the solutions of which are
\begin{equation*}
    \frac{a\T b + c\T d \pm \sqrt{(a\T b -c\T d)^2 + 4(a \T d)(b\T c)}}{2}.
\end{equation*}
\end{proof}

We are now ready to prove Lemma \ref{lemma:td_eigen_vals}.
\begin{proof}
We first consider standard TD($\lambda$).

We have that
\begin{equation*}
    (I - \stept X_t)(I - \stept X_t)\T = I + \step^2 X_t X_t\T -\step X_t -\step X_t\T.
\end{equation*}
Consider the matrix
\begin{equation*}
    A \doteq \step^2 X_t X_t\T -\step X_t -\step X_t\T = e_t(\step^2 (d_t\T d_t)e_t\T - \step d_t\T) - \step d_t e_t\T.
\end{equation*}
Using Lemma \ref{lemma:rank_2_eigen_vals}, the non-zero eigenvalues of $A$ are
\begin{equation*}
\begin{split}
    \lambda_1,\lambda_2 &= \frac{\step^2 \en^2 \dn^2 -2 \step \ed \pm \step \en \dn \sqrt{\step^2 \en^2 \dn^2 +4 - 4 \step \ed}}{2}.
\end{split}
\end{equation*}
Thus, the eigenvalues of $I+A$ are $1+\lambda_1,$ $1+\lambda_2$, and $k-2$ eigenvalues of 1.

We now consider implicit TD($\lambda$).
Since $Q_t X_t = Q_t e_t d_t\T$, the previous calculation holds, only with replacing $e_t$ with $Q_t e_t$.
Finally, let $\beta_t \doteq 1 - \frac{\step \en^2}{1+\step \en^2}$. By the Sherman-Morisson formula, $Q = (I - \frac{\step}{1+\step \en^2} e_t e_t\T)$, thus $e_t\T Q = \beta_t e_t\T$, $\Qed = \beta_t \ed$, and $\Qen = \beta_t^2 \en$, which gives the stated result.

Finally, assume that $|\lambda^+| < |\lambda^-|$. Then we must have $1 - \frac{\step^2 \en^2 \dn^2 -2 \step \ed}{2} < 0$, which means that $\step^2 \en^2 \dn^2 +4 - 4 \step \ed < 0$, making the discriminant of the quadratic equation negative, leading to a contradiction. We thus have $|\lambda^+| > |\lambda^-|$.
\end{proof}

\end{document}